%% file: anonymous-submission-latex-2026.tex
\documentclass[letterpaper]{article} % DO NOT CHANGE THIS
\usepackage{aaai2026}  % DO NOT CHANGE THIS
\usepackage{times}  % DO NOT CHANGE THIS
\usepackage{helvet}  % DO NOT CHANGE THIS
\usepackage{courier}  % DO NOT CHANGE THIS
\usepackage[hyphens]{url}  % DO NOT CHANGE THIS
\usepackage{graphicx} % DO NOT CHANGE THIS
\usepackage{natbib}  % DO NOT CHANGE THIS AND DO NOT ADD ANY OPTIONS TO IT
\usepackage{caption} % DO NOT CHANGE THIS AND DO NOT ADD ANY OPTIONS TO IT
\usepackage{algorithm}
\usepackage{algorithmic}

\nocopyright    
\usepackage{newfloat}
\usepackage{listings}
\DeclareCaptionStyle{ruled}{labelfont=normalfont,labelsep=colon,strut=off} % DO NOT CHANGE THIS
\floatstyle{ruled}
\newfloat{listing}{tb}{lst}{}
\floatname{listing}{Listing}
\title{Symmetry Breaking for Inductive Logic Programming}
\author{
    Andrew Cropper\textsuperscript{\rm 1,2}\equalcontrib, 
    David M. Cerna\textsuperscript{\rm 3,4}\equalcontrib, 
    Matti J\"{a}rvisalo\textsuperscript{\rm 2}\\
}
\affiliations{
        \textsuperscript{\rm 1}ELLIS Institute Finland\\
        \textsuperscript{\rm 2}University of Helsinki\\ 
 \textsuperscript{\rm 3}Dynatrace Research\\ \textsuperscript{\rm 4}Czech Academy of Sciences Institute of Computer Science\\
 
andrew.cropper@helsinki.fi, david.cerna@dynatrace.com, dcerna@cs.cas.cz, matti.jarvisalo@helsinki.fi

}

\input{preamble}

\begin{document}

\maketitle

\input{00-abstract}
\begin{links}
\link{Code}{https://github.com/logicand-learning-lab/aaai26-symbreak}
% \link{Extended version}{https://arxiv.org/pdf/2508.06263}
\end{links}
\input{01-intro}
\input{02-related}
\input{03-problem}
\input{04-algo}
\input{05-results}
\input{06-conclusions}
\input{07-acknowledgements}
\bibliography{kr-sample}
\newpage
\input{08-appendix}

\end{document}

%% file: preamble.tex
\usepackage{etoolbox}
\usepackage{fancyvrb} % More control over verbatim
\makeatletter
\preto{\@verbatim}{\small} % Change \small to \footnotesize, \scriptsize, etc., as desired
\makeatother

\usepackage{amsfonts}
\usepackage{footmisc}

\usepackage[utf8]{inputenc}
\usepackage{xcolor}
\usepackage{booktabs}
\usepackage{amsthm}
\usepackage{listings}
\usepackage{inconsolata}
\usepackage{tikz}
\usepackage{pgfplots}
\usepackage{amsmath}
\usepackage{microtype}
\usepackage{soul}
\usepgfplotslibrary{colorbrewer}
\pgfplotsset{compat = 1.15, cycle list/Set1-8} 
\usetikzlibrary{pgfplots.statistics,pgfplots.colorbrewer} 
\usepackage{pgfplotstable}
\usepackage{filecontents}

\usepackage{enumerate}%

\newcommand{\tw}[1]{\texttt{#1}}

\newcommand{\popper}{\textsc{Popper}}

\newcommand{\ordV}{\ensuremath{<_{\mathcal{V}}}}

\theoremstyle{definition}
\newtheorem{definition}{Definition}
\newtheorem{example}{Example}
\newtheorem{theorem}{Theorem}
\newtheorem{proposition}{Proposition}
\newtheorem{lemma}{Lemma}

\newtheorem{assumption}{Assumption}

\lstnewenvironment{myalgorithm}[1][] %defines the algorithm listing environment
{
    \lstset{ %this is the stype
        showspaces=false,               % show spaces adding particular underscores
        showstringspaces=false,         % underline spaces within strings
        mathescape=true,
        numbers=left,
        escapeinside={*}{*},
        columns=flexible,
        numbersep=2pt,        % default 10pt
        keywordstyle=\bfseries,
        keywords={,and, return, not, def, in, if, else, for, foreach, while, }
        numbers=left,
        xleftmargin=.0\textwidth,
        #1 % this is to add specific settings to an usage of this environment (for instnce, the caption and referable label)
    }
}
{}

%% file: 00-abstract.tex
\begin{abstract}
The goal of inductive logic programming is to search for a hypothesis that generalises training data and background knowledge.
The challenge is searching vast hypothesis spaces, which is exacerbated because many logically equivalent hypotheses exist.
To address this challenge, we introduce a method to break symmetries in the hypothesis space.
We implement our idea in answer set programming.
Our experiments on multiple domains, including visual reasoning and game playing, show that our approach can reduce solving times from over an hour to 17 seconds.
\end{abstract}

%% file: 01-intro.tex
\section{Introduction}
Inductive logic programming (ILP) is a form of machine learning \cite{mugg:ilp,ilpintro}. 
The goal is to search a hypothesis space for a hypothesis (a set of rules) that generalises given training examples and background knowledge (BK).
% ILP uses first-order logic to represent hypotheses, examples, and BK.

To illustrate ILP, consider the inductive reasoning game Zendo.
In this game, one player, the teacher, creates a secret rule that describes structures. 
The other players, the students, try to discover the secret rule by building structures. The teacher marks whether structures follow or
break the rule. 
The first student to correctly guess the rule wins.
For instance, for the positive and negative examples shown in Figure~\ref{fig:motivatingExample1}, a possible rule is \emph{``there is a small blue piece''}.
\begin{figure}[ht]
\centering
\includegraphics[width=.16\textwidth]{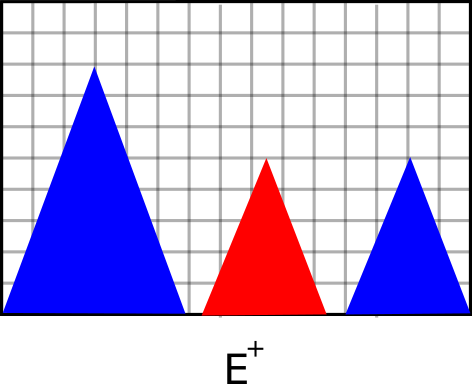}
\hspace{.4em}
\includegraphics[width=.16\textwidth]{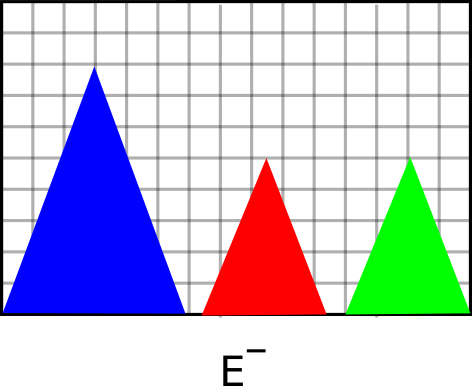}
    \caption{Positive (E$^+$ ) and negative (E$^-$) Zendo examples.}
    \label{fig:motivatingExample1}
\end{figure}

%\ac{@DC, can you add a zendo example similar to the NOPI paper?}
\noindent
Given these examples and BK about the structures, an ILP system can induce a rule such as:

\begin{center}
\begin{tabular}{l}
\emph{zendo(A) $\leftarrow$ piece(A,B), size(B,C), blue(B), small(C)}
\end{tabular}
\end{center} 

\noindent
The fundamental challenge in ILP is searching large hypothesis spaces.
Recent approaches tackle this challenge by delegating the search to an off-the-shelf solver, such as a Boolean satisfiability solver~\cite{atom,maxsynth} or an answer set programming (ASP) solver~\cite{ilasp,hexmil,popper,apperception}.

One problem for recent approaches is symmetries in the hypothesis space where many logically equivalent hypotheses exist.
For instance, consider these two Zendo rules:
\begin{center}
\begin{tabular}{l}
\emph{$r_1$ = zendo(A) $\leftarrow$ piece(A,B), size(B,C), blue(B), small(C)}\\
\emph{$r_2$ = zendo(A) $\leftarrow$ piece(A,C), size(C,B), blue(C), small(B)}
\end{tabular}
\end{center} 
\noindent
Both rules can be in the hypothesis space.
However, the rules are syntactic variants when considering theta-substitution \cite{plotkin:thesis,lloyd:book}.
For instance, $r_1 = r_2\theta$ where $\theta= \{C\mapsto B, B\mapsto C\}$.
However, determining whether two rules are symmetrical (syntactic variants) is hard. 
For instance, theta-substitution is in general NP-complete \cite{garey1979computers}.

In this paper, we address this symmetry challenge.
We first define the \emph{body-variant} problem: determining whether two rules only differ by the naming of body-only variables, such as $r_1$ and $r_2$ above.
We show that this problem is graph-isomorphism hard \cite{GIHardnessKST93}.

Due to the hardness of the body-variant problem, we introduce a sound, albeit incomplete, symmetry-breaking approach.
Our idea exploits an ordering of variables in a rule. 
If a body literal $l$ does not contain a variable $x$, but contains variables both larger and smaller than $x$ according to the variable ordering, then $x$ must occur in a literal lexicographically smaller than $l$. 
In other words, any gaps in the occurrence of variables with respect to an ordering must be justified by lexicographically smaller literals. We only order literals with arity greater than one because we assume that rules do not contain singleton variables, so the argument of a unary literal must appear elsewhere in the rule.

To illustrate our approach, consider rules $r_1$ and $r_2$ above.
Assuming that variables are ordered alphabetically, the literal \emph{piece(A,C)} of $r_2$ has a gap in its arguments because $B$ is missing. Furthermore,  \emph{size(C,B)} contains $B$ but $(A,C) <_{lex}(B,C)$. By renaming $C$ to $B$ and $B$ to $C$ within $r_2$ the rule $r_1$ results.  
Observe that $r_1$ contains both $piece(A,B)$ and $size(B,C)$ and $(A,B) <_{lex}(B,C)$. 
Thus, we can remove $r_2$ from the hypothesis space and keep $r_1$.
%Therefore, we can remove this gap by renaming $C$ to $B$ and % $C$ to $B$ resulting in the mapping $\theta$. 
% If we apply $\theta$ to $r_2$ the resulting rule is $r_1$ 
% which contains $piece(A,B)$ and $size(B,C)$ and $(A,B) <_{lex}(B,C)$ \ac{@DC, see discord}. 
% Thus, we can remove $r_2$ from the hypothesis space and keep $r_1$.

We implement our symmetry-breaking idea in ASP and demonstrate it with the ASP-based ILP system \popper{}, which can learn optimal and recursive hypotheses from noisy data \cite{popper,maxsynth}.
We show that adding symmetry-breaking constraints to \popper{} can drastically reduce solving and learning time, sometimes from over an hour to only 17 seconds.

\paragraph{Contributions and Novelty.}
The main novelty of this paper is a method for breaking symmetries in ILP.
% \ac{@AC, highlight how we combine two disparate fields}
As far as we are aware, there is no existing work on this topic.
% We expand on the novelty in Section \ref{sec:related}.
The impact, which we demonstrate on multiple domains, is vastly reduced solving and learning times.
Moreover, this work connects multiple AI areas, notably machine learning and constraint programming.
% , there is much potential for broad research to advance this idea.
Overall, we contribute the following:

\begin{itemize}
    \item We define the problem of determining whether two rules are body-variants (Definition \ref{def:body_variant_problem}).
    We show that this problem is graph isomorphism hard (Proposition \ref{prop:bodyVarHard}).
    \item We describe a symmetry-breaking approach to prune body-variant rules.
    We prove that this approach is sound and only prunes body-variant rules (Proposition \ref{lem:sound}).
    \item We describe an ASP implementation of our encoding that works with the ILP system \popper{}.
    \item We show experimentally on multiple domains, including visual reasoning and game playing, that our approach can reduce solving and learning times by 99\%.
\end{itemize}

%% file: 02-related.tex
\section{Related Work}
\label{sec:related}

\textbf{Rule learning.}
ILP induces rules from data, similar to rule learning methods \cite{DBLP:conf/ruleml/FurnkranzK15}, such as AMIE+ \cite{DBLP:journals/vldb/GalarragaTHS15} and RDFRules \cite{rdfrules}.
Rule-mining methods are typically limited to unary and binary relations and require facts as input and operate under an open-world assumption. 
By contrast, ILP usually operates under a closed-world assumption, supports relations of any arity, and can learn from definite programs as background knowledge.

\textbf{ILP.}
Most classical ILP approaches build a hypothesis using a set covering algorithm by adding rules one at a time \cite{foil,progol,tilde,aleph}.
These approaches build rules literal-by-literal using refinement operators to generalise or specialise a rule \cite{mis,foil,progol,claudien,tilde,aleph,DBLP:journals/ml/Tamaddoni-NezhadM09}.
% There is much research on refinement operators \cite{DBLP:journals/ml/Tamaddoni-NezhadM09}, such as identifying whether operators are \emph{complete} (generate all possible refinements),  \emph{proper} (each refinement is strictly more general or more specific),  and \emph{ideal} (complete and proper) \cite{claudien,ilp:book,luc:book}.
% There is much research on refinement operators \cite{}.
We differ because we do not use refinement operators.
 % build a hypothesis using a set covering algorithm and do not
% to build rules one-literal-at-a-time.
Rather, we break symmetries in an ILP approach that frames the ILP problem as a constraint satisfaction problem.

\textbf{Constraint-based ILP.}
Many ILP systems frame the ILP problem as a constraint satisfaction problem, such as SAT~\cite{atom,maxsynth} and ASP~\cite{aspal,ilasp,inspire,hexmil,apperception,aspsynth}.
Most of these approaches formulate the ILP problem as rule selection where they precompute every possible rule in the hypothesis space and use a solver to search for a subset that generalises the examples.
Because they precompute all possible rules, they cannot learn rules with many literals.
By contrast, we build constraints to restrict rule generation without precomputation.
% Moreover, because these approaches do not use refinement operators, there is little research on redundancy. 
% Our paper contributes to this issue.

\textbf{Redundancy in ILP.}
Symmetry breaking is a form of redundancy elimination.
There is much work on avoiding redundancy in ILP, such as when testing hypotheses on training data~\cite{queryoptimisation}.
\citet{RaedtR04} check whether a rule has a redundant atom before testing it on examples to avoid a coverage check.
\citet{quickfoil} prune rules with syntactic redundancy.
For the rule \emph{h(X) $\leftarrow$ p(X,Y), p(X,Z)}, they detect that \emph{p(X,Z)} duplicates \emph{p(X,Y)} under the renaming $Z \mapsto Y$, where $Z$ and $Y$ are not in other literals.
By contrast, we reduce redundancy through declarative symmetry-breaking constraints.

\textbf{Symmetry Breaking.}
Highly symmetric problems, when represented declaratively, can lead solvers such as ASP and SAT to perform redundant search.
Lex-leader~\cite{DBLP:conf/kr/CrawfordGLR96} uses tools for computing graph symmetries~\cite{nauty,DBLP:conf/alenex/JunttilaK07,DBLP:conf/dac/DargaLSM04} to automatically identify instance-specific symmetries.
Detected symmetries are eliminated by introducing lex-leader symmetry-breaking constraints during preprocessing.
This approach has been employed in SAT~\cite{DBLP:conf/dac/AloulMS03,DBLP:conf/sat/AndersBR24}, ASP~\cite{DBLP:journals/aicom/DrescherTW11,DBLP:journals/corr/Devriendt016}, and other declarative approaches.
\citet{alice_symmetry} use ILP to learn first-order ASP symmetry-breaking constraints from examples of instance-specific constraints. 
By contrast, we break symmetries in the ILP hypothesis space.
Our notion of symmetry differs from conventional domain-level symmetries~\cite{DBLP:journals/tplp/DevriendtBBD16}.
In our work, symmetry refers to syntactic variable interchangeability in rule bodies: two rules are symmetric if one can be obtained from the other by renaming variables.

%% file: 03-problem.tex
\section{Inductive Logic Programming}
\label{sec:ProblemSetting}
In this section, we define ILP and describe prerequisite notation.
Let $\mathcal{P}$ be a countably infinite set of predicate symbols and $\mathcal{V}$ be a countably infinite set of variables totally well-ordered by $<_{\mathcal{V}}$. A \emph{term} is either a constant or a variable. Every $p\in \mathcal{P}$ has an arity $0\leq a$ denoted $arity(p)$.
An \emph{atom} is of the form $p(t_1,\dots,t_a)$ where $p\in\mathcal{P}$, $t_1,\dots,t_a$ are terms. Given an atom $l= p(t_1,\dots,t_a)$, $sym(l) = p$, $arity(l)=a$, $args(l)=(t_1,\dots,t_a)$ where $(\dots)$ is an  ordered tuple.  We refer to an atom as \emph{ground} if all its arguments are constants and \emph{constant-free} if all its arguments are variables.  A literal is an atom or a negated atom. 
A rule $r$ is of the form $h\leftarrow p_1,\ldots, p_n$ where $h, p_1,\ldots, 
p_n$ are literals, $head(r) = h$, $body(r) = \{ p_1,\ldots, p_n\}$, and $body_{\geq 2}(r)$ is the subset of $body(r)$ containing all literals in $r$ with arity $\geq 2$. We will only consider rules that are \emph{head-connected}, i.e. all variables occurring in the head of a rule appear in a literal of the body of the rule. We refer to a rule as \emph{constant-free} if all of its literals are constant-free. 
By $vars(l)$ we denote the set of variables in a literal $l$. The variables of a rule $r$, denoted $vars(r)$, is defined as $head\_vars(r) \cup body\_vars(r)$ where $head\_vars(r)= vars(head(r))$ and $body\_vars(r)= vars(body(r))$.

%\subsection{Inductive Logic Programming}
We formulate our approach in the ILP learning from entailment setting \cite{luc:book}.
We define an ILP input:

\begin{definition}[\textbf{ILP input}]
\label{def:probin}
An ILP input is a tuple $(E, B, \mathcal{H})$ where $E=(E^+,E^-)$ is a pair of sets of ground atoms denoting positive ($E^+$) and negative ($E^-$) examples, $B$ is background knowledge, and $\mathcal{H}$ is a hypothesis space, i.e a set of possible hypotheses.
\end{definition}
\noindent
We restrict hypotheses and background knowledge to definite programs with the least Herbrand model semantics \cite{lloyd:book}.
We define a cost function:
\begin{definition}[\textbf{Cost function}]
\label{def:cost_function}
Given an ILP input $(E, B, \mathcal{H})$, a cost function $cost_{E,B} : \mathcal{H} \to \mathbb{N}$ assigns a numerical cost to each hypothesis in $\mathcal{H}$.
\end{definition}

\noindent
We define an \emph{optimal} hypothesis:
\begin{definition}[\textbf{Optimal hypothesis}]
\label{def:opthyp}
Given an ILP input $(E, B, \mathcal{H})$ and a cost function \emph{cost$_{E,B}$}, a hypothesis $h \in \mathcal{H}$ is \emph{optimal} with respect to \emph{cost$_{E,B}$} when $\forall h' \in \mathcal{H}$, \emph{cost$_{E,B}$}($h$) $\leq$ \emph{cost$_{E,B}$}($h'$).
\end{definition}

\section{Complexity of Symmetry Breaking in ILP}
\label{sec:complexity}

Our goal is to find a subset of a hypothesis space that contains at least one optimal hypothesis after symmetry breaking:
% , thereby maintaining soundness.
\begin{definition}[\textbf{Hypothesis space reduction problem}]
\label{def:hypredprob2}
Given an 
ILP input $(E, B, \mathcal{H})$, the \emph{hypothesis space reduction problem} is to find $\mathcal{H}' \subseteq \mathcal{H}$ such that if $\mathcal{H}$ contains an optimal hypothesis then there exists an optimal hypothesis $h \in \mathcal{H}'$.
% s.t. \emph{cost$_{E,B}$}($h$) = \emph{cost$_{E,B}$}($h'$).
\end{definition}
% \ac{@AC, mention symmetry breaking}
%\subsection{Hypothesis Variants}

\noindent
We focus on breaking symmetries by removing hypotheses that differ only by renaming variables.
We define a \emph{body-variant} rule:
\begin{definition}[\textbf{Body variant}]
A rule $r'$ is a \emph{body-variant} of a rule $r$ if there exists a bijective renaming $
\sigma$ from  $body\_vars(r)$ to $body\_vars(r')$ such that $r\sigma = r'$.
    % \ac{Copy from LP book?}
\end{definition}

\begin{example}
Consider the rules in the introduction:
% \begin{center}
\begin{tabular}{l}
\emph{$r_1$ = zendo(A) $\leftarrow$ piece(A,B), size(B,C), blue(B), small(C)}\\
\emph{$r_2$ = zendo(A) $\leftarrow$ piece(A,C), size(C,B), blue(C), small(B)}
\end{tabular}
% \end{center} 
Observe that  $head(r_1)=head(r_2)$ and using $\sigma_1=\{C\mapsto$  $ B, B\mapsto C\}$ and $\sigma_2= \{B\mapsto C,C\mapsto B\}$, it follows that $r_1\sigma_1 = r_2$ and $r_2\sigma_2 = r_1$.
\end{example}

\noindent
The \emph{body-variant} problem is to decide whether two rules differ only by the naming of body-only variables:
\begin{definition}[\textbf{Body-variant problem}]
\label{def:body_variant_problem}
 Given rules $r_1$ and $r_2$ such that $head(r_1)=head(r_2)$, the \emph{body-variant problem} is deciding whether $r_1$ and $r_2$ are body-variants of each other.

% there exists  bijective renamings $\sigma_1$ from $body\_vars(r_1)$ to $body\_vars(r_2)$ and $\sigma_2$ from $body\_vars(r_2)$ to $body\_vars(r_1)$ such that $r_1\sigma_1 = r_2$ and $r_2\sigma_2 = r_1$. 
%\ac{@DC, rephrase to simply say that r1 and r2 are body variants?}
\end{definition}
\noindent
Deciding whether two rules are body variants is intractable as there is a reduction from the \emph{graph isomorphism} problem to the body-variant problem, for which existing decision procedures are computationally prohibitive~\cite{graphIsoComplex,graphIsoComplex2}. 
The following result illustrates the (GI)-hardness~\cite{GIHardnessKST93} of the body-variant problem:

\begin{proposition}[\textbf{Body-variant hardness}]
\label{prop:bodyVarHard}
The body-variant problem is GI-hard\footnote{\citet{ArvindDKT15} reduce graph isomorphism to hypergraph isomorphism. Using this reduction we can extend our reduction from binary to n-ary predicates.}.
\end{proposition}
\begin{proof}[Proof (sketch)]
We encode a graph $G=(N,E)$  as a rule using a binary relation \textit{edge/2} and $|N|$ variables denoting nodes of $G$. The reduction follows from this encoding (See Appendix).
\end{proof}
\noindent
This result motivates us to develop an incomplete yet tractable approach to the body-variant problem.  
Before discussing our approach, we first generalise the body-variant problem to hypothesis variants:

\begin{definition}[\textbf{Hypothesis variant}]
A hypothesis $h'$ is a \emph{variant} of a hypothesis $h$ if there is a bijective mapping $f$ from the rules of $h'$ to the rules of $h$ such that for all $r\in h'$, $r$ is a body variant $f(r)$. 
%Given a hypothesis $h$, the set of variants of $h$ is denoted by \emph{variants(h)}.
% \ac{I think it needs to be stricter. H could have more runes than h'. I think it needs to be a one-to-one mapping}
\end{definition}
% \begin{example}
%     The following hypotheses are variants:
% \[
% \begin{array}{l}
%     h_1=\left\{
%     \begin{array}{l}        \emph{$r_1$ = z(A)$\leftarrow$piece(A,B),size(B,C),blue(B),small(C)}\\
% \emph{$r_2$ = z(A)$\leftarrow$piece(A,C),size(C,B),red(C),large(B)}
%     \end{array}
%     \right\}
%     \end{array}
% \]

\begin{example}
Consider the following hypotheses:
    \begin{center}
        \underline{Hypothesis 1 ($h_1$)}

\begin{tabular}{l}
\emph{$r_1$ = zendo(A) $\leftarrow$ piece(A,B), size(B,C), blue(B), small(C)}\\
\emph{$r_2$ = zendo(A) $\leftarrow$ piece(A,C), size(C,B), red(C), large(B)}
\end{tabular}
\end{center} 

    \begin{center}
    \underline{Hypothesis 2 ($h_2$)}
\begin{tabular}{l}
\emph{$r_3$ = zendo(A) $\leftarrow$ piece(A,C), size(C,B), blue(C), small(B)}\\
\emph{$r_4$ = zendo(A) $\leftarrow$ piece(A,B), size(B,C), red(B), large(C)}
\end{tabular}
\end{center} 
Using a mapping $f$ such that $f(r_1)=r_3$ and $f(r_2)=r_4$, we can observe that $h_1$ is a hypothesis-variant of $h_2$.
\end{example}
\noindent
We generalise the body-variant problem to hypotheses:

\begin{definition}[\textbf{Hypothesis-variant problem}]
\label{def:variant_problem}
 Given hypotheses $h_1$ and $h_2$, the \emph{hypothesis-variant problem} is deciding whether $h_1$ and $h_2$ are hypothesis-variants of each other.
 %bijective mapping $f$ from the rules of $h_1$ to the rules of $h_2$ such that for all $r\in h_1$, $r$ is a body variant $f(r)$.
  %\ac{@DC, rephrase to simply say that h1 and h2 are hypothesis variants?}
 % \ac{same as above}
\end{definition}
\noindent
The body-variant problem is a special case of the hypothesis-variant problem, where hypotheses contain a single rule.
Thus, it follows that the hypothesis-variant problem is also GI-hard. 

% \ac{@AC, SUMMARISE THE HYPOTHESIS-VARIANT PROBLEM SUCCINLTY}

\section{Tractable Symmetry Breaking for ILP}
\label{sec:tractableBreaking}
Due to the hardness of the hypothesis variant problem, we now describe a sound yet incomplete approach to remove hypothesis variants from an ILP hypothesis space. 
For simplicity, we describe the approach for a single-rule hypothesis before generalising to arbitrary hypotheses.
% \ac{@AC, make it clearer what the goal is}

Informally, our approach exploits variable  ordering in a rule and forces the following condition on all rules in the hypothesis space.
If a body literal $l$ does not contain a variable $x$, but contains variables both larger and smaller than $x$ according to the variable ordering, then $x$ must occur in a literal lexicographically smaller than $l$. 
In other words, any skipped occurrence of a variable, with respect to $\ordV $, must be justified by lexicographically smaller literals. 

To illustrate the idea, reconsider this Zendo rule from the introduction:

% \begin{center}
\begin{tabular}{l}
\hspace{-1em}\emph{$r_2$ = zendo(A) $\leftarrow$ piece(A,C), size(C,B), blue(C), small(B)}
\end{tabular}
% \end{center} 

\noindent
We order variables alphabetically, i.e. $A \ordV B \ordV C$. 
The variable $B$ is not in \emph{piece(A,C)} so we need to check if a lexicographically smaller variable tuple than $(A,C)$ is in the rule. 
To compare tuples we first order the variables in the tuple and add a prefix containing occurrences of the smallest variable. 
We assume that rules do not contain singleton variables. 
Thus, the argument of a unary literal must appear elsewhere in the rule, i.e. how the literal is ordered is completely dependent on another literal.
For instance, the tuple of \emph{size(C,B)} is $(B,C)$ and the tuple of \emph{piece(A,C)} is $(A,C)$. Observe that $(A,C)<_{lex} (B,C)$ and there are no tuples smaller than $(A,C)$ derivable from the rule. Thus, $r_2$ is pruned because it does not contain a tuple smaller than $(A,C)$ containing the skipped variable $B$. Observe that a similar analysis shows that $r_1$ from the introduction is not pruned. 
% \begin{center}
% \begin{tabular}{l}
% size(B,C), blue(B)}
% \end{tabular}
% \end{center} 

% \noindent
To formalise our approach, we make several assumptions explicit:

\begin{itemize}
\item[(i)] A rule contains only variables, no constants. For example, the rule \emph{zendo(A)\ $\leftarrow$ piece(A,B)}  obeys this assumption, but \emph{zendo(A)\ $\leftarrow$ piece(A,lbt)} does not, where \emph{lbt} is a constant.
\item[(ii)] The variables in the head of a rule are always the smallest variables according to the variable ordering\footnote{The total order $\ordV$ on variable symbols. %See Section~\ref{sec:ProblemSetting}.
%\ac{@DC, weird comment as we are in S3}
}. 
Formally, for any  rule $r$, $v_1\in head\_vars(r)$, $v_2\in body\_vars(r)$, $v_1\ordV v_2$. 
For example, the rule \emph{zendo(A)\ $\leftarrow$ piece(A,B)} obeys this assumption as $A\ordV B$ while the rule \emph{zendo(B)\ $\leftarrow$ piece(A,B)} does not. 

\item[(iii)] A variable is in a rule if and only if every smaller variable is in the rule. 
Formally, for any  rule $r$,  $v_1,v_2\in vars(r)$, if there exists $v_3\in \mathcal{V}$ such that $v_1\ordV v_3 \ordV v_2$, then $v_3\in vars(r)$. 
For example, the rule \emph{zendo(A)\ $\leftarrow$ piece(A,B)} obeys this assumption, while the rule \emph{zendo(A)\ $\leftarrow$ piece(A,C)} does not as $A\ordV B \ordV C$. 
\end{itemize}

\noindent
Any rule that violates assumptions (ii) and/or (iii)     can be transformed by a variable renaming into the appropriate form, either by shifting variables or swapping head and body variables. Furthermore, we consider the argument tuples of literals reordered with respect to $\ordV$:

\begin{definition}[\textbf{Ordered variables}]
\label{def:orderd_vars}
Let $l$ be a literal with arguments $args(l)=(x_1,\dots,x_n)$.
Let $(i_1, \ldots, i_n)$ be a permutation of $(1, \ldots, n)$ such that $x_{i_1} \le_V x_{i_2} \le_V \cdots \le_V x_{i_n}$.
Then the ordered variables of $l$ are $ord\_vars(l) = (x_{i_1}, \ldots, x_{i_n})$.

% $1\leq i_1,\cdots,i_n\leq n$ such that for all $1\leq j< n$, $x_{i_j}\ordnsV x_{i_{j+1}}$ and $i_j\not = i_{j+1}$. Then $ord\_vars(l)= (x_{i_1},\dots,x_{i_n})$. 
\end{definition}
\begin{example}
    Let $A \ordV B \ordV D$, and $l=p(D,A,B)$ be a literal. Then $ord\_vars(l) = (A,B,D)$.
\end{example}
\noindent
Additionally, we add a prefix to the tuples resulting from Definition~\ref{def:orderd_vars} to produce tuples of uniform length:
\begin{definition}[\textbf{Prefix padding}]
\label{def:prefix_padding}
% \ac{for def 10 on pre\_pad, I’d put the “max … times” below, not above, reads a bit confusingly when above}
Let $k\geq 0$ and $l$ be a literal such that $ord\_vars(l)= (x_{1},\dots,x_{n})$. 
Then $$\mathit{pre\_pad}_k(l) =(\hspace{-1,1em}\underbrace{x_{\star},\dots,x_{\star},}_{\max\{0,k-n\}\mbox{ times}}\hspace{-1,1em}x_1,\dots,x_n) $$
where $x_{\star}$ is the minimal element with respect to  $\ordV$.  \end{definition}
\noindent For the rest of this section we assume the following:
\begin{assumption}
% \ac{change name//}
  Let $p$ be a predicate symbol in the $\mathit{BK}$ such that for all other predicate symbols $q$ in the $\textit{BK}$, $arity(p)\geq arity(q)$. Then $k\geq arity(p)$.
\end{assumption}
\noindent
We define a lexicographical literal order to order literals by first rearranging their arguments with respect to $\ordV$ and then adding a prefix to the resulting tuples to produce tuples of uniform size:

\begin{definition}[\textbf{Lexicographical literal order}]
\label{def:arg_ordering}
Let  $l_1$ and $l_2$ be literals with arity $\geq 2$. Then we say that $l_1<_{lex}^k l_2$ if
$ \mathit{pre\_pad}_k(l_1) <_{lex} \mathit{pre\_pad}_k(l_2)$
where $<_{lex}$ is the lexicographical order on $k$-tuples.
\end{definition}
\begin{example}
    Let $l_1=p(D,A,B)$ and $l_2=q(C,B)$ be literals where variables are alphabetically ordered and $k=3$. Then $l_2<_{lex}^3 l_1$ because  $ord\_vars(l_1) = (A,B,D)$, $ord\_vars(l_2) = (B,C)$, $\mathit{pre\_pad}_3(l_1)= (A,B,D)$, $\mathit{pre\_pad}_3(l_2)= (x_*,B,C)$, and $(x_*,B,C) <_{lex} (A,B,D)$ where $x_*= A$.
\end{example}

\noindent 
We use this order to identify literals with \emph{skipped} variables:

\begin{definition}[\textbf{Skipped}]
\label{def:var_skipped}
Let $l$ be a literal such that $\mathit{pre\_pad}_k(l) = (x_{1},\dots,x_{n})$. Then $skipped_k(l) =\{ y\mid  x_{1}\ordV y\ordV x_{n} \wedge y\not\in vars(l)\}$.
\end{definition}
\begin{example}
\label{ex:skipped}
    Consider the rule $h(A,B)\leftarrow p(A,E),$ $p(B,C),p(C,D)$
where variables are ordered alphabetically. Observe that $skipped_2(p(A,E)) = \{B,C,D\}$ and $skipped_2(p(C,D)) =\emptyset$.
\end{example}
% \noindent We identify skipped variables in a literal as \emph{gaps} when they occur in a rule:
% \begin{definition}[\textbf{Gaps}]
% \label{def:var_gaps}
% Let $r$ be a  rule, and $v\in \mathit{body\_var}(r)$. Then $gaps_k(v,r) = \{l\mid l\in body_{\geq 2}(r)\wedge  v\in skipped_k(l)\}$.
% \end{definition}
% \dc{TODO: remove gaps and update the proof to use witnessed.}
% \begin{example}
%     Consider the rule from Example~\ref{ex:skipped}. Observe that  $gaps(C,r_1)=gaps(D,r_1)= \{p(A,E)\}$
% \end{example}

\noindent 
Rules can contain literals with skipped variables but the skipped variables must be witnessed by a literal lower in the lexicographical literal order: 
\begin{definition}[\textbf{Witnessed}]
\label{def:witnessed}
Let $r$ be a  rule, 
$v$ be a variable,  
$l_1\in body_{\geq 2}(r)$ such that  $v\in skipped_k(l_1)$, 
and $l_2\in body_{\geq 2}(r)$ such that $v\in vars(l_2)$ and $l_2<_{lex}^k l_1$.  
Then we say that $l_1$ is $v$-witnessed in $r$.
\end{definition}

\noindent 
A variable is \emph{unsafe} if it is skipped in a literal and there is no witnessing literal:
% \noindent When there is a literal that skips a variable, with respect to the given rule, and there does not exists a witnessing literal, we refer to that variable as \emph{unsafe}: 
% \ac{
% A variable is \emph{unsafe} in a rule when it is in a literal that is not witnessed:
% }
\begin{definition}[\textbf{Safe variable}]
\label{def:unsafe_safe_var}
Let $r$ be a rule and $v\in body\_vars(r)$ such that for all $l\in body_{\geq 2}(r)$, where  $v\in skipped_k(l)$, $l$ is $v$-witnessed in $r$. Then $v$ is \emph{safe}. Otherwise, $v$ is \emph{unsafe}.
\end{definition}
% \noindent When all body variables of a rule are \emph{safe} we refer to the rule as \emph{proper}:
% \begin{definition}[\textbf{Proper rule}]
% \label{def:proper} 
% A rule is \emph{proper} if all its variables are safe. \dc{replace proper with rule containing only safe variables.}
% \end{definition}
\noindent We illustrate the concept of safe variables: 
\begin{example}
    Consider the rules:\\
    \begin{tabular}{l}
\emph{$r_2$ = h(A,B) $\leftarrow$ p(A,C),p(B,E),p(C,D)}\\
\emph{$r_3$ = h(A,B) $\leftarrow$ p(A,C),p(B,D),p(C,E)}
\end{tabular}
% where variables are ordered alphabetically. 

\noindent
Observe that $p(B,E)$ skips $D$ and none of the literals of $r_2$ witness $p(B,E)$. The only literal containing $D$ is $p(C,D)$ and $p(B,E)<_{lex}^2 p(C,D)$. Thus, $D$ is \emph{unsafe} in $r_2$. In rule $r_3$,  $p(C,E)$ skips $D$ and $p(C,E)$ is witnessed by $p(B,D)$ because $p(B,D)<_{lex}^2 p(C,E)$. Thus, $D$ is safe in $r_3$. Observe that in $r_3$ all variables are safe including $C$ and $E$, i.e. $p(A,C)<_{lex}^2 p(B,D)$.  
\end{example}

% \ac{@AC, rephrase}

% \noindent
% We now show that every optimal hypothesis has a proper hypothesis-variant.

\noindent
We now show that every rule has a body-variant containing only safe variables:

\begin{proposition}[\textbf{Soundness}]
\label{lem:sound}
For every rule $r$ there exists a rule $r'$ such that $r'$ is a body-variant of $r$ and all variables in $r'$ are safe.
\end{proposition}
\begin{proof}[Proof (sketch)]To simplify our argument we assume that $vars(r) = \{x_1,\cdots, x_m\}$ where for $1\leq j< m$, $x_j\ordV x_{j+1}$. We prove the proposition by induction, selecting the smallest unsafe variable $x$ (according to $\ordV$) and then constructing a substitution that, when applied to $r$, results in a rule where the smallest unsafe variable is larger than $x$. After finitely many steps the smallest unsafe variable is larger than $x_m$, i.e. every variable in the constructed rule is safe.
\end{proof}

\noindent See the appendix for a full proof. Below we outline the transformation used by the induction step of the full proof. 
\begin{example}
\label{ex:soundnessvariants}
Let $$r_8= h(A,B)\leftarrow p(A,E),p(B,\mathbf{C}),p(C,D).$$ where the variables are ordered alphabetically.
The argument outlined in the proof of  Proposition~\ref{lem:sound} applies to rule $r_8$ as follows.
Observe that $C$ is the smallest unsafe variable in $r_8$. Applying  $\sigma_1 = \{E\mapsto C, C\mapsto F\}\{F\mapsto E\}$ to $r_8$ we get the rule:
$$r_9= h(A,B)\leftarrow p(A,C),p(B,E),p(E,\mathbf{D}).$$
Now $D$ is the smallest unsafe variable in $r_9$. Applying
$\sigma_2 = \{E\mapsto D, D\mapsto F\}\{F\mapsto E\}$ to $r_9$ we get the rule:
$$r_{10}= h(A,B)\leftarrow p(A,C),p(B,D),p(D,E).$$
Observe that rule $r_{10}$ does not have unsafe variables. Thus applying the substitution $\sigma =\sigma_1\sigma_2= \{E\mapsto C, C\mapsto D, D\mapsto E\}$ to $r$ results in a body-variant where all variables are safe. 
\end{example}

\noindent
We now show our main result:
\begin{theorem}
Let 
$(E, B, \mathcal{H})$ be an ILP input, \emph{cost$_{E,B}$} be a cost function,
and $h \in \mathcal{H}$ be an optimal hypothesis with respect to \emph{cost$_{E,B}$}.
Then there exists $h' \in \mathcal{H}$ such that $h'$ is a variant of $h$ and all rules of $h'$ only contain safe variables.
\end{theorem}

% \ac{by considering bijections between rules with the same size and predicate symbols\footnote{This additional step is  performed by our ASP implementation.}. }

\noindent
See the appendix for examples of rules that are body-variants of each other and only contain safe variables.
% and examples of unsound extensions of the order we considered. We leave further improvements to future work. 

%% file: 04-algo.tex
\section{Implementation}
\label{sec:Implmenation}
% Although our symmetry-breaking idea is general and system agnostic, 
We demonstrate our idea in ASP with the ILP system \popper{} \cite{popper,combo}.
% in which the symmetry breaking can be implemented via ASP constraints.

\subsection{Popper}
Our symmetry-breaking approach directly works with all variants of \popper{hopper,maxsynth,propper}.
For simplicity, we describe the basic version of \popper{}.

\popper{} uses a generate, test, combine, and constrain loop to find an optimal hypothesis (Definition \ref{def:opthyp}) (the full algorithm is in the appendix).
% \ref{alg:popper} shows the \popper{} algorithm.
\popper{} starts with an answer set program $\mathcal{P}$.
This program can be viewed as a \emph{generator} program because each model (answer set) represents a hypothesis.
The program $\mathcal{P}$ uses head (\emph{hlit}/3) and body (\emph{blit/3}) literals to represent a hypothesis.
The first argument of each literal is the rule ID, the second is the predicate symbol, and the third is the literal variables, where \emph{0} represents \emph{A}, \emph{1} represents \emph{B}, etc.
% \begin{center}
% \begin{tabular}{l}
% \end{tabular}
% \end{center} 
% \noindent
For instance, 
\popper{} represents the rule \emph{f(A,B) $\leftarrow$ tail(A,C), head(C,B)} as a set with three atoms:
\begin{center}
\begin{verbatim}
{hlit(0,f,(0,1)), blit(0,tail,(0,2)), 
 blit(0,head,(2,1))}
\end{verbatim}
\end{center}
% \begin{center}
% \begin{tabular}{l}
% \emph{\{hlit(0,last,(0,1)), blit(0,tail,(0,2)), blit(0,head,(2,1))\}}
% \end{tabular}
% \end{center} 

\noindent
The program $\mathcal{P}$ contains choice rules for head and body literals:
\begin{verbatim}
{hlit(Rule,Pred,Vars)}:- 
  rule(Rule), vars(Vars,Arity),  hpred(Pred,Arity).
{blit(Rule,Pred,Vars)}:- 
  rule(Rule), vars(Vars,Arity), bpred(Pred,Arity).
\end{verbatim}

\noindent
The literal \emph{rule(Rule)} denotes rule indices.
The literals \emph{hpred(Pred,Arity)} and \emph{bpred(Pred,Arity)} denote predicate symbols and arities that may appear in the head or body of a rule, respectively.
The literal \emph{vars(Vars,Arity)} denotes all possible variable tuples.

In the \emph{generate stage}, \popper{} uses an ASP system to find a model of $\mathcal{P}$ for a fixed hypothesis size, enforced via a cardinality constraint on the number of head and body literals.
If no model is found, \popper{} increments the hypothesis size and loops again.
If a model exists, \popper{} converts it to a hypothesis $h$ (a definite program).

In the \emph{test stage}, \popper{} uses Prolog to test $h$ on the training examples and background knowledge.
% If $h$ is a solution, \popper{} returns it.
If $h$ entails at least one positive example and no negative examples, \popper{} saves $h$ as a \emph{promising program}.

In the \emph{combine stage}, \popper{} searches for a combination (a union) of promising programs that entails all the positive examples and has minimal size.
\popper{} formulates the search as a combinatorial optimisation problem \cite{combo}, implemented in ASP as an optimisation problem.
If a combination exists, \popper{} saves it as the best hypothesis and updates the maximum hypothesis size.
% \popper{} does not save a program as promising if it is recursive or has predicate invention (line 10). The reason is that a combination of recursive programs or programs with invented predicates can entail more examples than the union of the examples entailed by each program. However, \popper{} can learn hypotheses with recursion or predicate invention as they can be output by the generate stage (line 6) and evaluated (line 9).

In the \emph{constrain stage}, \popper{} uses $h$ to build constraints, which it adds to $\mathcal{P}$ to prune models and thus prune the hypothesis space.
For instance, if $h$ does not entail any positive example, \popper{} adds a constraint to prune its specialisations as they are guaranteed not to entail any positive example.
For instance, the following constraint prunes all specialisations (supersets) of the rule \emph{f(A,B) $\leftarrow$ tail(A,C), head(C,B)}:

\begin{center}
\begin{verbatim}
:- hlit(R,f,(0,1)), blit(R,tail,(0,2)), 
   blit(R,head,(2,1)).
\end{verbatim}
\end{center} 

\noindent
\popper{} repeats this loop using multi-shot solving \cite{multishot-clingo} and terminates when it exhausts the models of $\mathcal{P}$ or exceeds a user-defined timeout.
It then returns the best hypothesis found.

% \begin{algorithm}[t!]
% \small
% {
% \begin{myalgorithm}[]
% def $\text{popper}$(bk, E+, E-, max_size):
%   cons = {}
%   promising = {}
%   best_hypothesis = {}
%   while True:
%     h = generate(cons, max_size)
%     if h == UNSAT:
%       return best_hypothesis
%     tp, fp = test(E+, E-, bk, h)
%     if tp > 0 and fp == 0:
%       promising += h
%       combine_outcome = combine(promising, max_size)
%       if combine_outcome != NO_SOLUTION:
%         best_hypothesis = combine_outcome
%         max_size = size(best_hypothesis)-1
%     cons += constrain(h, tp, fp)
%   return best_hypothesis
% \end{myalgorithm}
% \caption{
% \popper{}
% }
% \label{alg:popper}
% }
% \end{algorithm}

\subsection{Symmetry Breaking Encoding}
\label{subsec:SymBreakEncode}
We now describe our ASP encoding $\mathcal{E}$ to break symmetries. 
We add $\mathcal{E}$ to the answer set program $\mathcal{P}$ used by \popper{} to generate programs to prune unsafe rules and thus unsafe hypotheses.

% Let $\mathcal{P}$ be a set of predicate symbols that may appear in a rule, $k$ be the maximum number of variables allowed in a rule, and $m$ be the maximum arity of any predicate symbol in $P$.
The fact \tw{var\_member(Vars,V)} denotes that variable \tw{V} is a member of  variable tuple \tw{Vars}, e.g. \tw{var\_member((0,4,3),3)}.
For every variable tuple \tw{xs}, we sort the tuple to \tw{ys} and add the fact \tw{ordered\_vars(xs,ys)} to $\mathcal{E}$.
For instance, for the variable tuple \tw{(4,1,3)}, we add the fact
\tw{ordered\_vars((4,1,3),(1,3,4))}.
These facts match \emph{ordered variables} (Definition 
 \ref{def:orderd_vars}).
We add facts to encode the lexicographic order over variable tuples (Definition 
\ref{def:arg_ordering}).
For instance, we add the facts
\tw{lower((0,0,1),(0,0,2))} and \tw{lower((4,7,1),(4,8,2))}.
We add all skipped facts of the form \tw{skipped(Vars,V)} to denote that the variable \tw{V} is strictly between two variables \tw{A} and \tw{B} in an ordered variable tuple \tw{Vars}, where \tw{V} is not in \tw{Vars} (Definition \ref{def:var_skipped}); these include,
for instance, the facts \tw{skipped((0,1,3),2)}, \tw{skipped((1,3,5),2)}, and \tw{skipped((1,3,5),4)}.
We add a rule to identify the ordered variable tuple of a selected body literal:
\begin{verbatim}
appears(Rule,OrderedVars):- 
  blit(Rule,_,Vars), padded_vars(Vars,PaddedVars),
  ordered_vars(PaddedVars,OrderedVars).
\end{verbatim}

% \noindent
% We add a rule to identify gaps (Definition \ref{def:var_gaps} in variable tuples:
% \begin{verbatim}
% gap(Vars,V):-
%     ordered_vars(_, Vars),
%     var_member(A,Vars),
%     var_member(B,Vars),
%     var(V),
%     not var_member(V,Vars),
%     A < V < B.
% \end{verbatim}
% For instance, we add the facts \tw{gap((4,6,7),5)} and \tw{gap((1,2,3,5),4)}.
% \ac{@ALL, this gap rule can be ground as a preprocessing step and could just say we add the facts. Do you think it would clearer?}

\noindent
% \ac{@AC MAKE IT CLEAR THAT THE HEAD VARIABLES ARE ALWAYS SAFE}
The literal \tw{padded\_vars(Vars,PaddedVars}) performs prefix padding on variable tuples (Definition \ref{def:prefix_padding}).
For instance, assuming that the maximum arity of any literal is 4, for the variable tuple \tw{(4,1)} we added the fact \tw{padded\_vars((4,1),(0,0,4,1))}.
We add a rule to identify witnessed variables (Definition \ref{def:witnessed}):
\begin{verbatim}
witnessed(Rule,V,Vars1):-    
   appears(Rule,Vars1), skipped(Vars1,V),
   lower(Vars2,Vars1), var_member(V,Vars2),
   appears(Rule,Vars2). 
\end{verbatim}
\noindent
Finally, we add a constraint to prune rules with unsafe variables \ref{def:unsafe_safe_var}):
\begin{verbatim}
:- body_var(Rule,V), appears(Rule,Vars), 
   skipped(Vars,V), not witnessed(Rule,V,Vars).
\end{verbatim}

\noindent
Our symmetry-breaking encoding adds at most $O(m \cdot n^2 \cdot k)$ ground rules to the ASP solver, where $m$ is the number of candidate rules in a hypothesis, $n$ is the number of possible variable tuples, and $k$ is the number of possible variables.

%% file: 05-results.tex
\section{Experiments}
To test our claim that pruning variants can reduce solving time, our experiments aim to answer the question:

\begin{description}
\item[Q1] Can symmetry breaking reduce solving time?
\end{description}

\noindent
To answer \textbf{Q1}, we compare the solving time of \popper{} with and without symmetry breaking.

To test our claim that pruning variants allows us to scale to harder tasks, our experiments aim to answer the question:

\begin{description}
\item[Q2] Can symmetry breaking reduce solving time when progressively increasing the complexity of tasks?
\end{description}

\noindent
To answer \textbf{Q2}, we compare the solving time of \popper{} with and without symmetry breaking when varying the complexity of an ILP task.

Finally, the goal of breaking symmetries is to reduce solving time and, in turn, reduce learning time.
Therefore, our experiments aim to answer the question:

\begin{description}
\item[Q3] Can symmetry breaking reduce learning time?
\end{description}

\noindent
To answer \textbf{Q3}, we compare the learning time of \popper{} with and without symmetry breaking.

\subsubsection{Experimental Setup}
To answer \textbf{Q1}, we compare the solving time (time spent generating hypotheses) of \popper{} with and without symmetry breaking.
% We measure solving time as the time 
We use a solving timeout of 20 minutes per task.
% We round times over one second to the nearest second.
% We use \popper{} with at most \ac{?} variables of rules of size at most \ac{?}.
% We also measure the balanced accuracy of the learned hypotheses, but this dependent variable is not 
To answer \textbf{Q2}, we compare the solving time of \popper{} with and without symmetry breaking on progressively harder ILP tasks by increasing the number of variables allowed in a rule\footnote{
\citet{popper} show that the hypothesis space grows exponentially in the number of variables allowed in a rule.}.
We use a solving timeout of 60 minutes per task.
% We enforce a solving timeout of 60 minutes per task.
For \textbf{Q1} and \textbf{Q2}, the independent variable is whether \popper{} uses symmetry breaking and the dependent variable is the solving time, which depends entirely on the independent variable.
To answer \textbf{Q3}, we compare the learning time of \popper{} with and without symmetry breaking.
We use a learning timeout of 60 minutes per task.
However, learning times are a function of many things, not only solving time.
For instance, symmetry breaking could reduce solving time by 50\% allowing \popper{} to generate twice as many hypotheses but it also needs to test them.
Therefore, learning time does not directly measure the impact of symmetry breaking.
In other words, the dependent variable (learning time) does not entirely depend on the independent variable.

% We measure mean learning time.
In all experiments, the runtimes include the time spent generating the symmetry-breaking rules.
We round times over one second to the nearest second.
We repeat all experiments 10 times.
We plot and report 95\% confidence intervals (CI).
We compute  95\% CI via bootstrapping when data is non-normal.
% We use a paired t-test or a Wilcoxon Signed-Rank Test (depending on whether the differences are normally distributed) to determine the statistical significance of any differences in the results, and any subsequent reference to a significance test refers to one of these two tests.
% We report results with 95\% confidence intervals.
To determine statistical significance, we apply either a paired t-test or the Wilcoxon signed-rank test, depending on whether the differences are normally distributed. 
We use the Benjamini–Hochberg procedure to correct for multiple comparisons.
We use \popper{} 4.4.0.
We use an AWS m6a.16xlarge instance to run experiments where each learning task uses a single core.

% \paragraph{Reproducibility.}
% The code and experimental data for reproducing the experiments are available as supplementary material and will be made publicly available if the paper is accepted for publication.    

\subsubsection{Domains}
We use 449 learning tasks from several domains:
% The appendix provides additional information about our domains and tasks.

\textbf{1D-ARC.} This dataset \cite{onedarc} contains visual reasoning tasks inspired by the abstract reasoning corpus \cite{arc}.

% \textbf{Alzheimer.} These real-world tasks 
% \cite{DBLP:journals/ngc/KingSS95} involve learning rules describing four properties desirable for drug design against Alzheimer’s disease.

\textbf{IGGP.} The task is to induce rules from game traces \cite{iggp} from the general game playing competition \cite{ggp}.

\textbf{IMDB.}
A real-world dataset which contains relations about movies \cite{imdb}. 

\textbf{List functions.} The goal of each task in this data \cite{ruleefficient} is to identify a function that maps input lists to output lists, where list elements are natural numbers. 

% \textbf{Satellite.}
% The task is to learn diagnostic rules for battery faults in the power subsystem of a satellite, which consists of 40 components and 29 sensors \cite{satellite}.

\textbf{Trains.}
The goal is to find a hypothesis that distinguishes east and west trains \cite{michalski:trains}.

\textbf{Zendo.}
An inductive game where players discover secret rules by building structures \cite{discopopper}.

% \subsubsection{Balanced accuracy}
% Each task contains training and testing examples.
% We use the training examples to train \popper{}, i.e. to learn a hypothesis.
% We test a hypothesis on the testing examples.
% Given a hypothesis $h$, background knowledge $B$, and a set of examples $T$, a \emph{true positive} is a positive example in $T$ entailed by $h \cup B$ and a \emph{true negative} is a negative example in $T$ not entailed by $h \cup B$.
% We denote the number of true positives and true negatives as $tp_{T}(h)$ and $tn_{T}(h)$ respectively. 
% We measure the \emph{balanced predictive accuracy} of a hypothesis as:
% % This measure handles imbalanced data by evaluating the average performance across both positive and negative classes:
% \begin{align*}
% ba(h) = \frac{1}{2}\left( \frac{tp_{T}(h)}{tp_{T}(h)+fn_{T}(h)}+\frac{tn_{T}(h)}{tn_{T}(h)+fp_{T}(h)} \right)
% \end{align*}
% % Balanced accuracy is defined as the average of the recall (or true positive rate) for each class.
% % Balanced accuracy is equivalent to standard accuracy when the hypothesis performs equally well on both classes or when the data is balanced.

% % We assume a set $T$ of unseen test data in the following.

\subsection{Experimental Results}

\input{05-q1}
\input{05-q2}
\input{05-q3}

% \ac{@DC, I commented out your text because of compile errors}
% \ac{@AC, I fixed the important part. It is just an examples illustrating why predicate ordering does not mix with the gap constraint. }

% \dc{Below is an example of a pair of programs which are considered distinct by the gap constraint, but are variants of each other:
% \begin{align*}
%     f(0,1,2)\mbox{:-}& in(0,1,5), in(0,6,2), succ(1,\mathbf{4}),\\ & succ(\mathbf{3},6), empty(0,3), empty(0,4).\\\\
%     f(0,1,2)\mbox{:-}& in(0,1,5), in(0,6,2), succ(1,\mathbf{3}),\\ & succ(\mathbf{4},6), empty(0,3), empty(0,4).
% \end{align*}
% }
% \begin{verbatim}
% Gap constraint cannot prune: 

% body_literal(0,in,3,(0,1,5))
% body_literal(0,in,3,(0,6,2))
% body_literal(0,my_succ,2,(1,4))  <---
% body_literal(0,my_succ,2,(3,6))  <---
% body_literal(0,empty,2,(0,3)) 
% body_literal(0,empty,2,(0,4)) 
% Answer:
% 2
% body_literal(0,in,3,(0,1,5)) 
% body_literal(0,in,3,(0,6,2)) 
% body_literal(0,my_succ,2,(1,3)) <---
% body_literal(0,my_succ,2,(4,6)) <---
% body_literal(0,empty,2,(0,3)) 
% body_literal(0,empty,2,(0,4)) 
% \end{verbatim}

%% file: 05-q1.tex
\subsubsection{Q1. Can Symmetry Breaking Reduce Solving Time?}

Figure \ref{fig:q1_popper} shows the solving times of \popper{} with and without symmetry breaking\footnote{
Detailed per-task results and corresponding improvements are in the appendix.}.
Significance tests confirm $(p < 0.05)$ that symmetry breaking reduces solving times on 97/449 (22\%) tasks and increases solving times on 1/449 (0\%) tasks.
There is no significant difference in the other tasks.
The mean decrease in solving time is $178 \pm 56$ seconds.
The median decrease is 20 seconds with 95\% CI between 8 and 51 seconds.
The mean and median increase is $41$ seconds.
% The median increase is 20 seconds where the 95\% CI is between 8 and 51 seconds.
Some improvements are substantial.
For instance, for the \emph{sokoban-terminal} task, symmetry breaking reduces the solving time from $1075 \pm 112$ seconds to $59 \pm 6$ seconds.
These are minimum improvements because \popper{} without symmetry breaking often times out after 20 minutes.
With a longer timeout, we would likely see greater improvements.
Overall, the results show that symmetry breaking can drastically reduce solving time.

\begin{figure}[ht!]
\centering
\includegraphics[scale=1.3]{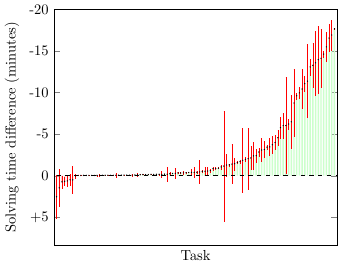}
\caption{Solving time difference (minutes) with symmetry breaking.
The tasks are ordered by improvement.
}
\label{fig:q1_popper}
\end{figure}

%% file: 05-q2.tex
\subsubsection{Q2. Can Symmetry Breaking Reduce Solving Time When Progressively Increasing the Complexity of Tasks?}

Figure \ref{fig:q2} shows the solving times of \popper{} with and without symmetry breaking on progressively harder ILP tasks.
The results show that \popper{} without symmetry breaking struggles to scale to harder tasks.
Without symmetry breaking, the mean solving times when allowed to use 7, 8, or 9 variables in a rule are $125 \pm 44$, $1411 \pm 310$, and $3600 \pm 0$ seconds respectively.
By contrast, with symmetry breaking, the mean solving times when allowed to use 7, 8, or 9 variables in a rule are $4 \pm 0$, $7 \pm 0$, and $17 \pm 2$ seconds respectively.
In other words, for the hardest task, symmetry breaking reduces solving time from over an hour to only 17 seconds, a 99.5\% reduction.
Overall, the results show that symmetry breaking can reduce solving time when progressively increasing the complexity of tasks.

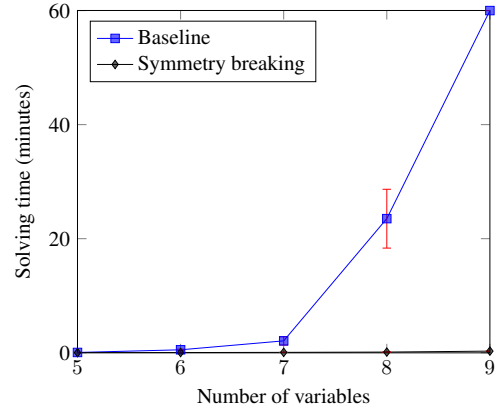
\begin{figure}[ht!]
\centering
\begin{tikzpicture}[scale=.8]
\begin{axis}
[   
    xmin=5, 
    xmax=9, 
    ymin=-1,
    ymax=3601,
    % ymode=log,
    xtick = {5, 6, 7, 8, 9},
    ytick={0,1200,2400,3600},
    yticklabels={0,20,40,60},  
    legend pos=north west,
    legend cell align=left, 
    % ymode=log,
    ylabel={Solving time (minutes)},
    xlabel={Number of variables},
    xlabel style={font=\scriptsize},
    ylabel style={font=\scriptsize},
    xlabel near ticks, ylabel near ticks
]
\addplot [mark=square*,blue, fill opacity=0.6, error bars/.cd, y dir=both, y explicit, error bar style={color=red}] table [x=n, y=mean, y error minus=minus, y error plus=plus] {data/q2-times-popper.txt};
\addplot [mark=diamond*,black, fill opacity=0.6, error bars/.cd, y dir=both, y explicit, error bar style={color=red}] table [x=n, y=mean, y error minus=minus, y error plus=plus] {data/q2-times-sbc.txt};

\addlegendentry{Baseline}
\addlegendentry{Symmetry breaking}
\end{axis}
\end{tikzpicture}  
\caption{Solving times (minutes) of \popper{} with and without (baseline) symmetry breaking on one \emph{trains} task.
We vary the number of variables allowed in a rule and thus the size of the hypothesis space.
% The errors bars denote 95\% confidence intervals.
}
\label{fig:q2}
\end{figure}

%% file: 05-q3.tex
\subsubsection{Q3. Can Symmetry Breaking Reduce Learning Time?}

Figure \ref{fig:q3_popper} shows the learning times of \popper{} with and without symmetry breaking.
Significance tests confirm $(p < 0.05)$ that symmetry breaking reduces learning times on 128/449 (29\%) tasks and increases learning times on 12/449 (3\%) tasks.
There is no significant difference in the other tasks.
The mean decrease in learning time is $385 \pm 109$ seconds and the median is 54 seconds with 95\% CI between 21 and 134 seconds.
The mean increase is $192 \pm 121$ seconds and the median is 120 seconds with 95\% CI between 54 and 266 seconds.
These are minimum improvements because \popper{} without symmetry breaking often times out after 60 minutes.
With a longer timeout, we would likely see greater improvements.
Overall, the results show that symmetry breaking can drastically reduce learning time.

\begin{figure}[ht!]
\centering
\includegraphics[scale=1.3]{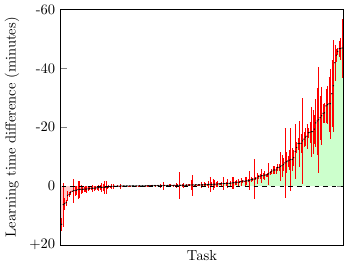}
\caption{Learning time difference (minutes) with symmetry breaking.
The tasks are ordered by improvement.
}
\label{fig:q3_popper}
\end{figure}

%% file: 06-conclusions.tex
\section{Conclusions and Limitations}
We introduced a symmetry-breaking method for ILP.
We showed that determining whether two rules are body-variants (Definition \ref{def:body_variant_problem}) is graph isomorphism hard (Proposition \ref{prop:bodyVarHard}).
We described a symmetry-breaking approach to prune body-variant rules and implemented it in ASP.
We proved that this approach is sound and only prunes body-variant rules (Proposition \ref{lem:sound}).
We have experimentally shown on multiple domains, including visual reasoning and game playing, that our approach can drastically reduce solving time and learning time, sometimes from over an hour to only 17 seconds.
% \ac{One thing we could emphasize a bit more regarding the results is that the approach gives significant performance improvements especially on hard tasks, thereby improving the scalability of Popper}
\subsubsection{Limitations}
Our symmetry-breaking approach is incomplete, as complete symmetry-breaking is graph isomorphism hard (Proposition \ref{prop:bodyVarHard}).
% Nonetheless, there may still be practical ways to break more symmetries.
Finding other efficient symmetry-breaking techniques is future work.

%% file: 07-acknowledgements.tex
\section*{Acknowledgments}
We thank Andreas Niskanen for early discussions about this
work. Andrew Cropper was supported by his EPSRC fellowship (EP/V040340/1). David M. Cerna was supported by the
Czech Science Foundation Grant 22-06414L and Cost Action
CA20111 EuroProofNet. Matti Jarvisalo was supported by Research Council of Finland (grant 356046).

%% file: 08-appendix.tex
\newpage
\section{Appendix}

\noindent 
In this section, we define the notation used in this paper, provide a short introduction to ILP, and define the \emph{body-variant problem} and discuss its complexity.
\paragraph{Preliminaries.}
Let $\mathcal{P}$ be a countably infinite set of predicate symbols totally well-ordered by $<_{\mathcal{P}}$, and $\mathcal{V}$ be a countably infinite set of variables totally well-ordered by $<_{\mathcal{V}}$.
%, and $\mathcal{C}$ be a countably infinite set of constants totally well-ordered by $<_{\mathcal{C}}$. 
A \emph{term} is either a constant or a variable. Every $p\in \mathcal{P}$ has an arity $0\leq a$ denoted $arity(p)$.
An \emph{atom} is of the form $p(t_1,\dots,t_a)$ where $p\in\mathcal{P}$, $t_1,\dots,t_a$ are terms. Given an atom $l= p(t_1,\dots,t_a)$, $sym(l) = p$, $arity(l)=a$, $args(l)=(t_1,\dots,t_a)$ where $(\dots)$ is an  ordered tuple.  We refer to an atom as \emph{ground} if all its arguments are constants and \emph{constant-free} if all its arguments are variables.  A literal is an atom or a negated atom. 
A rule $r$ is of the form $h\leftarrow p_1,\ldots, p_n$ where $h, p_1,\ldots, 
p_n$ are literals, $head(r) = h$, and $body(r) = \{ p_1,\ldots, p_n\}$. We will only consider rules that are \emph{head-connected}, i.e., all variables occurring in the head of a rule appear in a literal of the body of the rule. We refer to a rule as \emph{constant-free} if all of its literals are constant-free. Given a rule $r$, we define $max\_arity(r) = \max \{arity(l) \mid l \in body(r)\}$
By $vars(l)$ we denote the set of variables in a literal $l$. The variables of a rule $r$, denoted $vars(r)$, is defined as $head\_vars(r) \cup body\_vars(r)$ where $head\_vars(r)= vars(head(r))$ and $body\_vars(r)= vars(body(r))$.

A {\it substitution} is a function $\sigma$ from variables to terms such that $\sigma(x)\neq  x$ for only finitely many variables. The set of the variables that are not mapped to themselves is called the \emph{domain} of $\sigma$, denoted as $dom(\sigma)$. The \emph{range} of $\sigma$, denoted $ran(\sigma)$, is the set of terms $\{\sigma(x) \mid x\in dom(\sigma) \}$. A substitution $\sigma$ is a \emph{renaming} if $ran(\sigma)\subset \mathcal{V}$. The set of renamings is denoted by $\mathcal{V}_{rm}$. Substitutions are extended to rules and literals in the usual manner. We use the postfix notation for substitution application to terms and write $t\sigma$ instead of $\sigma(t)$.

\subsection{Body-variant Problem Hardness}
In~\cite{ArvindDKT15}, the authors provide a reduction from graph isomorphism to hypergraph isomorphism. Using this reduction we can extend the following reduction from binary to n-ary predicates.
\newcommand\propcnt{\value{proposition}}
\setcounter{proposition}{0}
\begin{proposition}[\textbf{Body-variant hardness}]
The body-variant problem is GI-hard.
\end{proposition}
\begin{proof}
Let $(N_1,E_1)$ and $(N_2,E_2)$ be finite undirected graphs such that $|N_1|=|N_2|$. We construct rules $r_1$ and $r_2$ as follows:
\begin{itemize}
\item $head(r_1)=head(r_2) = h$ where $h$ has arity 0.
\item $W_1,W_2\subset \mathcal{V}$, sets of variables, such that $W_1\cap W_2=\emptyset$, $|W_1|=|N_1|$, and $|W_2|=|N_2|$
\item $f_1: N_1\rightarrow W_1$ and $f_2: N_2\rightarrow W_2$ are bijections.
\item $body(r_1) = \{edge(f_1(n_1),f_1(n_2))\mid (n_1,n_2)\in E_1\}$ 
\item $body(r_2) = \{edge(f_2(n_1),f_2(n_2))\mid (n_1,n_2)\in E_2\}$ 
\end{itemize}
Observe, $edge/2$ is a predicate symbol denoting an edge in a graph. Now, if there exists bijective renamings $\sigma_1:W_1\rightarrow W_2$ and $\sigma_2:W_2\rightarrow W_1$ such that $r_1\sigma_1=r_2$ and $r_2\sigma_2=r_1$, Then we can deduce that
\begin{itemize}
\item $(N_1,E_1)=  (N_1,\{ (f_1^{-1}(x\sigma_2),f_1^{-1}(y\sigma_2)) \mid edge(x,y)\in body(r_2)\})$
\item $(N_2,E_2)=  (N_2,\{ (f_2^{-1}(x\sigma_1),f_2^{-1}(y\sigma_1)) \mid edge(x,y)\in body(r_1)\})$
\end{itemize}
and thus the graphs $(N_1,E_1)$ and $(N_2,E_2)$ are isomorphic.    
\end{proof}

\subsection{Soundness of safe rule Symmetry Breaking} Below we provide the full proof of soundness. We first define two transformations to transform arbitrary rules into rules containing only safe variables and show that these transformation do not change the ordering of literals with respect to the lexicographical literal order.

\noindent Our first result allows us to swap a variable in a literal with a larger variable without changing the order of the effected literals:

\begin{lemma}
\label{prop:sub1}
Let  $l_1$ and $l_2$ be literals with arity $\geq 2$, $x_1\in vars(l_1)$, $x_2\in vars(l_2)$, $x_2\not \in vars(l_1)$, $x_2\ordV x_1 \ordV y$, and $\theta = \{x_1\mapsto x_2,x_2\mapsto y\}$. Then  if $l_1 <_{lex}^k l_2$, then $ l_1\theta <_{lex}^k l_2\theta$.
\end{lemma}
\begin{proof}
Observe that $\mathit{pre\_pad}_k(l_1\theta)<_{lex}\mathit{pre\_pad}_k(l_1) $ and $\mathit{pre\_pad}_k(l_2)<_{lex}\mathit{pre\_pad}_k(l_2\theta)$, thus it follows that if $l_1 <_{lex}^k l_2$, then $ l_1\theta <_{lex}^k l_2\theta$.
\end{proof}
\begin{example}
Consider the following rules:
    $$r_4: h(A,B)\leftarrow p(A,E),p(B,C),p(C,D)$$
         $$r_5: h(A,B)\leftarrow p(A,C),p(B,E),p(E,D)$$
where  $r_5= r_4\sigma$ and $\sigma =\{E\mapsto C, C\mapsto E\}$. Observe that $p(A,E) <_{lex}^2 p(B,C)$ and $p(A,E)\sigma <_{lex}^2 p(B,C)\sigma$.
\end{example}

% \ac{need some text here, \ac}

\noindent 
Our second result allows us to shift variables downwards in a literal to remove gaps in the body variables of the given rule:
\begin{lemma}
\label{prop:sub2}
Let $r$ be a  rule such that $body\_vars(r)= \{x_1,\cdots,x_{i-1},x_{i+1},\dots,x_m\}$ where for $1\leq j< m$,  $x_j\ordV x_{j+1}$, and $\sigma = \{x_{j+1} \mapsto x_{j} \mid i\leq j< m\}$.  Then  for all $l_1,l_2\in body_{\geq 2}(r)$, if $l_1 <_{lex}^k l_2$, then $ l_1\sigma <_{lex}^k l_2\sigma$.
\end{lemma}
\begin{proof}
Follows from the fact that for all literals $l\in body_{\geq 2}(r)$, if  $\mathit{pre\_pad}_k(l)=(y_1,\cdots, y_k)$, then $\mathit{pre\_pad}_k(l\sigma)=(y_1\sigma,\cdots, y_k\sigma)$, i.e. elements do not swap places after substitution. 
\end{proof}
\begin{example}
Consider the following rules:
    $$r_6: h(A,B)\leftarrow p(A,C),p(B,F),p(F,D)$$
         $$r_7: h(A,B)\leftarrow p(A,C),p(B,E),p(E,D)$$
where  $r_7= r_6\sigma$ and $\sigma =\{F\mapsto E\}$. Observe that $p(B,F) <_{lex}^2 p(F,D)$ and $p(B,F)\sigma <_{lex}^2 p(F,D)\sigma$.
\end{example}
Using these two transformations we show that every rule has a  body-variant containing only safe variables: 
\setcounter{proposition}{1}
\begin{proposition}[\textbf{Soundness}]
%\label{lem:sound}
For every rule $r$ there exists a rule $r'$ such that $r'$ is a body-variant of $r$ and all variables in $r'$ are safe.
\end{proposition}
\begin{proof} To simplify our argument we assume that $vars(r) = \{x_1,\cdots, x_m\}$ where for $1\leq j< m$, $x_j\ordV x_{j+1}$ and $x_{m+1}$ is a variable such that $x_m \ordV x_{m+1}$. We prove the proposition by well-founded induction\footnote{Induction is performed over a finite total ordering towards a maximum element.}. Let us assume that $x_i$ is the smallest unsafe variable. Let $l_1 = \min \{ l\mid l\in body_{\geq 2}(r)\ \wedge\ x_i\in skipped_k(l)\}$, with respect to the $<_{lex}^k$, and $i< w\leq m$ such that $x_{w}\in vars(l_1)$ and for all $i<j<w$, $x_{j}\not \in vars(l_1)$. Now we build the substitution 
$$\sigma_i = \{x_{w}\mapsto x_i, x_i\mapsto x_{m+1}\}\{x_{j+1} \mapsto x_{j} \mid w\leq j\leq m\}.$$
Observe, $x_i\in vars(l_1\sigma_i)$, and for all $l_2\in body_{\geq 2}(r)$ where $x_i\in vars(l_2)$, $l_2\sigma_i\in body_{\geq 2}(r)$ such that  $x_i\in skipped_k(r\sigma_i)$. By Definition~\ref{def:unsafe_safe_var}, we know that  for all $l_2\in body_{\geq 2}(r)$ where $x_i\in vars(l_2)$, $l_1<_{lex}^k l_2$. Thus we can derive that (i) for all $l_2\in body_{\geq 2}(r)$ where $x_i\in vars(l_2)$, $l_1\sigma_i <_{lex}^k l_2\sigma_i$, using Lemma~\ref{prop:sub1}\ \&~\ref{prop:sub2} , and thus (ii) $x_i$ is safe in $r\sigma_i$. Note,  $x_m=\max \{x\mid x\in vars(r\sigma_i)\}$ and for all unsafe variables $x$ in $r\sigma_i$, $x_{i}\ordV x$.

Observe that this construction builds a bijective renaming $\sigma = \sigma_i\dots\sigma_{m-1}$ such that (i) $r\sigma$ is a body variant of $r$ and (ii) $r\sigma$ only contains safe variables, i.e. $r\sigma$ only contains safe variables.
\end{proof}

\subsection{Safe body-variants}
 Observe that rules containing only safe variables may be body-variants of each other. 
\setcounter{proposition}{\propcnt}
\begin{proposition}
 There exist a rule $r$ containing only safe variables and a body variant of $r$ that only safe variables. 
\end{proposition}
\begin{proof}
Consider the rules
\begin{center}
\begin{tabular}{l}
\emph{r$_{11}$  = h(A,B) $\leftarrow$ p(B,D), p(C,E), p(A,C), p(A,D)}\\
\emph{r$_{12}$  = h(A,B) $\leftarrow$ p(B,C), p(D,E), p(A,C), p(A,D)}
\end{tabular}
\end{center}
Both r$_{11}$ and r$_{12}$ are contain only safe variables and $r_{12} =r_{11}\sigma$ where  $\sigma = \{C\mapsto D, D\mapsto C\}$. 
\end{proof}
\subsection{Monadic and Safe Symmetry Breaking}

One may consider introducing further symmetry breaking by introducing ordering constraints on predicate symbols. For instance, we can refer to a rule $r$ as  \emph{monadic safe} if for all $p(x),q(y)\in body(r)$ where $x,y\not \in head\_vars(r)$, $p<_{\mathcal{P}}q$ and $x\leq_{\mathcal{V}}y$. Consider the following rule where $<_{\mathcal{P}}$ and $<_{\mathcal{V}}$ follow alphabetic ordering:

\begin{center}
\begin{tabular}{l}
\emph{$r_{13}$  =   h(A)$\leftarrow$ w(A,B), q(B), m(B,C),  s(C), p(C).
}
\end{tabular}
\end{center}
Observe that $r_{13}$ is only contains safe variables, but it is not \emph{monadic safe} as  $p<_{\mathcal{P}}q$ and $B<_{\mathcal{V}}C$. If we rename variables to make the rule monadic safe (replace $B$ with $C$ and vice versa) , then the rule could contain unsafe variables (B is unsafe in $w(A,C)$). Thus, naively adding symmetry breaking based on predicate ordering does not result in a sound symmetry breaker. While it is possible to fix the above constraint by adding additional restrictions, doing so increases the complexity of computing whether a rule is \emph{monadic safe}, and thus, reducing the efficacy of the symmetry breaking with respect to brute force enumeration.